%% file: wafr2018.tex
\documentclass{styles/llncs}
\usepackage[pdftex]{graphicx}
\usepackage{fancyhdr}
\usepackage[subpreambles=true]{standalone} 
\usepackage{tikz}
\usepackage{array}
\usepackage{amsfonts}
\usepackage{amsmath} 
\usepackage{mathbbol}
\usepackage[ruled]{algorithm}
\usepackage{algpseudocode}
\usepackage{subcaption}
\usepackage{wrapfig}
\let\llncssubparagraph\subparagraph
\let\subparagraph\paragraph
\usepackage[compact]{titlesec}
\let\subparagraph\llncssubparagraph
\titlespacing{\section}{0pt}{4ex}{2ex}
\titlespacing{\subsection}{0pt}{2ex}{1ex}
\titlespacing{\subsubsection}{0pt}{1ex}{0ex}

\newenvironment{dm}
  {\vspace*{-2pt}\displaymath}{\vspace*{-2pt}\enddisplaymath}
\newenvironment{eq}
  {\vspace*{-2pt}\equation}{\vspace*{-2pt}\endequation}
\newcommand{\sss}{\scriptscriptstyle}
\renewenvironment{proof}{\noindent{\bf Proof.}}{\hfill$\Box$\medskip}

\usepackage{url}

\input{wafr18-macros}

\begin{document}
\mainmatter              % start of a contribution
\title{Guided Exploration of Human Intentions for Human-Robot Interaction}
\titlerunning{}  % abbreviated title (for running head)
%                                     also used for the TOC unless
%                                     \toctitle is used
%
%\author{Ivar Ekeland\inst{1} \and Roger Temam\inst{2}
%Jeffrey Dean \and David Grove \and Craig Chambers \and Kim~B.~Bruce \and
%Elsa Bertino}

\author{Min Chen \and David Hsu \and Wee Sun Lee}
\authorrunning{Min Chen et al.} % abbreviated author list (for running head)
%
%%%% list of authors for the TOC (use if author list has to be modified)
%\tocauthor{Ivar Ekeland, Roger Temam, Jeffrey Dean, David Grove,
%Craig Chambers, Kim B. Bruce, and Elisa Bertino}
%
\institute{National University of Singapore, Singapore 117417, Singapore}
%\email{I.Ekeland@princeton.edu},\\ WWW home page:
%\texttt{http://users/\homedir iekeland/web/welcome.html}
%\and
%Universit\'{e} de Paris-Sud,
%Laboratoire d'Analyse Num\'{e}rique, B\^{a}timent 425,\\
%F-91405 Orsay Cedex, France}

\maketitle              % typeset the title of the contribution

\begin{abstract}

  Robot understanding of human intentions is essential for fluid human-robot
  interaction.  Intentions, however, cannot be directly observed and must be
  inferred from behaviors. We learn a model of adaptive human behavior
  conditioned on the intention as a latent variable.  We then embed the human
  behavior model into a principled probabilistic decision model, which
  enables the robot to (i) explore actively in order to infer human intentions
  and (ii) choose actions that maximize its performance. Furthermore, the
  robot learns from the demonstrated actions of human experts to further
  improve exploration.  Preliminary experiments in simulation indicate that
  our approach, when applied to autonomous driving, improves the efficiency
  and safety of driving in common interactive driving scenarios.
    % \keywords{probabilistic reasoning, human-robot interaction}
\end{abstract}
\section{Introduction}
\label{sec:introduction}

% Understand human intention is important 
Understanding human intentions is essential for fluid human-robot interaction.  It
helps the robot to interpret humans' behaviors and predict their future
actions.  Earlier work often treats humans as passive moving obstacles in the
environment: humans do not react to robot
actions~\cref{bai2015intention,bandyopadhyay2013intention,fern2007decision}.
At a result, the robot is overly conservative. It waits and observes until 
humans' intentions become clear.  However, in reality, humans are not merely
moving obstacles, and they respond to robot actions.  For example, in our study
of lane-switch for autonomous driving, a robot car tries to switch to a
lane, in which a human-driven car drives
(\figref{fig:lane-switch-example}).  A conservative human driver will slow
down, while an aggressive driver will accelerate and refuse to let the robot
car in.  If the robot car chooses to wait, the human driver will maintain
the speed, and the robot car will not learn the human driver's intention.
 Human intentions, human actions, and robot actions are
interconnected.
% there is a correlation among human actions, human intentions, and robot actions.
The robot must take advantage of the connections and actively 
\emph{explore} in order to
understand human intentions.

\begin{figure}[t]
    \centering
    \includegraphics[width=1.0\columnwidth]{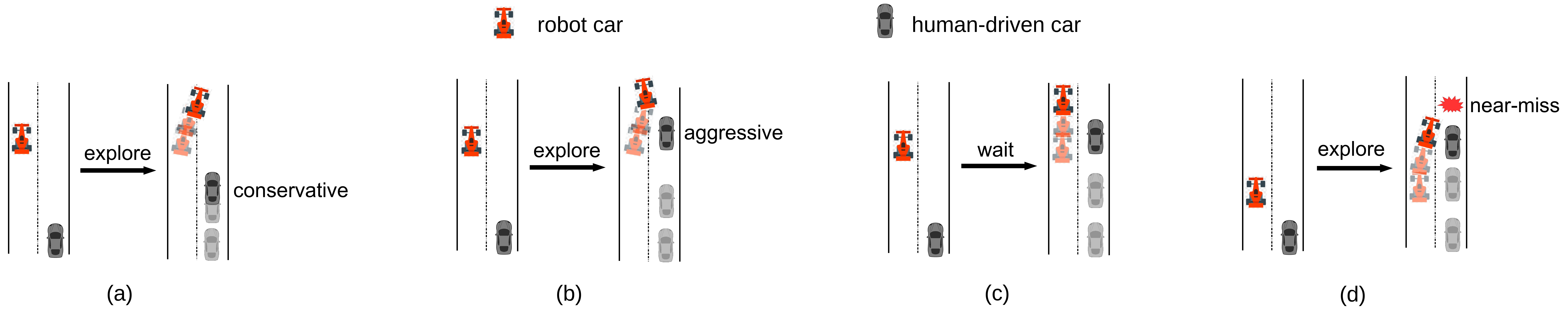}
    \caption{
        A robot car explores human intentions during the lane-switch. 
        (a,b)~If the robot car tries to explore, 
        a conservative human driver will slow down, while
        an aggressive human driver may accelerate, revealing the human intentions. 
        (c)~If the robot car chooses to wait,
        the human driver will maintain the speed. The robot car will not learn
        the human driver's intention and initiate the lane switch suitably.
        (d)~Exploration may be 
        dangerous, when the two car are too close.
    }
    \label{fig:lane-switch-example}
\end{figure}

Exploration is, however, not always appropriate. 
% On the other hand, arbitrary robot exploration might be 
% inappropriate in certain scenarios.
Consider the lane-switch  example again (\figref{fig:lane-switch-example}d).
If the robot car switches lanes when the two cars are very close, the
human will not have enough time to react.  In our study, the participants
indicated that they felt unsafe in such situations.  So the robot must not
only explore, but also explore safely, for effective human-robot interaction.
% Clearly, the way robot explores affects human's perception of safety, and
% consequently, the interaction.

% Contribution: A computational model that explicitly models
% human intention. Futhermore, we introduce a safe heuristic
% that helps generate safe robot exploration actions.
%Our goal in this work is to give robot car capabilities to
%plan actions that are efficient and safe at the same time.
%Our contributions are two-fold.

To this end, we propose two ideas.  The first is an intention-driven human
behavior model,  integrated into a probabilistic robot decision making
framework for human-robot interaction.  Unlike earlier work on intentional
behavior modeling~\cref{bai2015intention}, our model conditions human actions
explicitly on robot actions and captures \emph{adaptive} human behaviors.
Since the human intention is not directly observable, we model it as a latent
variable in a partially observable Markov decision process
(POMDP)~\cref{kaelbling1998planning}.  We further assume that the human
intention remains static during a single interaction, thus reducing the POMDP
model to a computationally more efficient variant,
POMDP-lite~\cref{chen2016pomdp}.  Despite the simplifying assumption, the
resulting intention POMDP-lite model successfully captures many interesting
human behaviors. See \secref{sec:experiments} for examples.  Handcrafting
accurate POMDP models is a major challenge.  Here we take a data-driven
approach and learn the intention POMDP-lite model from data.  Solving an
intention POMDP-lite model produces a policy that enables the robot to explore
actively and infer human intentions in order to improve robot performance.
% But, intention POMDP-lite does not contain
% information on human's intent to be probed by the robot,
% and may generate explorative actions that are too aggressive.

Aggressive exploration is sometimes unsafe in applications such as driving. 
Our second idea is to leverage human expert demonstrations for improved robot
exploration. 
% Our intuition is that a human expert would know how
% to probe another human in a way that is gentle and 
% effective, and this can be used to guide robot exploration.
We learn from human demonstrations a probability distribution over state-action
pairs. The learned distribution captures the actions of human experts when
exploration is needed. It is then used as a heuristic to guide robot
exploration and favor the frequently demonstrated state-action pairs.

% Empirical results 
We evaluated our approach in simulation on common driving
tasks, including lane-switch and intersection navigation.  Compared with a
myopic policy that does not actively explore human intentions, intention
POMDP-lite substantially improved robot driving efficiency.  Combined with
guided exploration, it also improved driving safety.  While our
experiments are specific to autonomous driving, the approach is general and
applicable to other human-robot interaction tasks that require
exploration of human intentions.
%that require safe exploration of human intentions. 

In the following, \secref{sec:related} briefly surveys related
work. \secref{sec:intention-pomdp-lite} presents the intention POMDP-lite
model and guided exploration. \secref{sec:humanpolicy-safe} describes how we
learn an intention-driven human driving policy and a distribution for guided exploration.
\secref{sec:experiments} compares our approach with common alternatives
in simulation. \secref{sec:discussion} discusses the limitations of the
current approach and directions for further investigation.

\section{Related Work}
\label{sec:related}

Intention has been studied extensively in the field
of psychology~\cref{astington1993child,bratman1987intention},
where intention is characterized as a mental state that 
represents human's commitment to carrying out a sequence
of actions in the future.
Understanding intentions is crucial in understanding
various social contexts, 
\eg, it helps to interpret other people's behaviors and 
predict their future actions~\cref{feinfield1999young}. 
%We believe intention is also important for the robot
%to understand human behaviors in human-robot
%interactive tasks.

% Behavior modeling
In human-robot interaction, intention has been used as a 
means to model
human behaviors~\cref{nikolaidis2015efficient,bai2015intention,bandyopadhyay2013intention,sunberg2017value,lam2015improving}.
For example, Bai \etal~\cref{bai2015intention} modeled
pedestrians' behaviors with a set of intentions, 
and they enabled the autonomous car to drive successfully 
in a crowd.
However, they assumed that the human won't react to 
the robot, which made the robot act conservatively
most of the time.
Most recently, Sadigh \etal~\cref{sadigh2016information} 
showed that robot's action directly affects human actions,
which can be used to actively infer human intentions~\cref{sadigh2016planning}.
However, inferring human intentions is not the end goal.
Instead, our work embeds an intention-driven human behavior model 
into a principled robot decision model to maximize the robot performance.
In other words, 
the robot may choose to ignore the human 
if he/she does not affect robot performance.

% Robot decision making & Guided exploration 
%In a human-robot interaction task,
%To achieve that,
To maximize performance,
the robot needs to actively infer human 
intentions (exploration), and achieve its own goal (exploitation).
In addition, the robot needs to explore gently 
as the human might not willing to be probed
in certain scenarios.
%, \eg, \figref{fig:lane-switch-example} (right column).
The partially observable Markov decision process (POMDP)~\cref{kaelbling1998planning} 
trades off exploration and exploitation optimally.
However, POMDP itself does not model 
human's intent to be probed, thus it may generate 
explorative actions that are too aggressive.
Imitation learning derives a robot policy directly from human
demonstrations~\cref{abbeel2004apprenticeship}.
But the robot policy cannot be generalized to unseen
state space.
Instead of learning a robot policy directly,
Garcia \etal~\cref{garcia2012safe} used human demonstrations
to guide robot explorations,
and they significantly reduced the damage incurred from 
exploring unknown state-action space.
Our work draws insight from~\cref{garcia2012safe},
and we explicitly guide robot explorations in a POMDP model
with human demonstrations.

\section{Intention POMDP-lite with Guided Exploration}
\label{sec:intention-pomdp-lite}

\subsection{Mathematical Formulation of Human-Robot Interaction}
Mathematically, we formulate the human-robot interaction
problem as a Markov decision process 
$\tuple{\X, \AR, \AH, T, \RwdR{}, \RwdH{}, \gamma}$,
where $\x{} \in \X$ is the world state.
$\AR$ is the set of actions that the robot can take, 
and $\AH$ is 
the set of actions that the human can take.
The system evolves according to a stochastic state transition
function
$T(\x{}, \ar{}, \ah{}, \x{}') = P(\x{}'|\x{},\ar{},\ah{})$,
which captures the probability of transitioning to
state $\x{}'$ when joint actions $\tuple{\ar{}, \ah{}}$
are applied at state $\x{}$.
At each time step,
the robot receives a reward of $\rwdR{}(\x{}, \ar{}, \ah{})$
and the human receives a reward of $\rwdH{}(\x{}, \ar{}, \ah{})$.
%when joint actions \tuple{\ar{}, \ah{}} are applied at state \x{}.
The discount factor $\gamma$ is a constant scalar that
favors immediate rewards over future ones.

Given a human behavior policy, \ie,
$\ah{} \sim \policyh$, the optimal
value function of the robot is given by Bellman's equation 
\begin{eq}
    V^*(\x{} | \policyh) =
    \max_{\ar{}} \bigg{\{}
    \underset{\ah{} \sim \policyh}{\ev} \bigg{[}
    \rwdR{}(\x{}, \ar{}, \ah{}) +
    \sum_{\x{}'}
    \gamma P(\x{}' | \x{}, \ar{}, \ah{}) V^*(\x{}' | \policyh)
    \bigg{]}
    \bigg{\}}
    \label{eq:opt-value}
\end{eq}

%However, in our case, the robot does not know the human
%policy in advance. Its optimal value function can
%be computed by taking expectation over the human policy

%\begin{eq}
    %V^*(\x{}) =
    %\underset{\policyh}{\ev}
    %\big{[} V^*(\x{} | \policyh) \big{]}
    %\label{eq:opt-value-expect}
%\end{eq}

The optimal robot policy \policyropt is the action
that maximizes the right hand size of~\equref{eq:opt-value},
and the key to solve it is to have a model
of human behaviors.

%The expected discounted return of the robot when starting
%from state $\x{0}$ is

%\begin{eq}
    %\vr{}(\x{0}|\policyr, \policyh) =
    %\underset{\ar{t} \sim \policyr, \ah{t} \sim \policyh} {\mathrm{E}}~
    %\sum_{t=0}^{\infty}
    %\gamma^t \rwdR{}(\x{t},\ar{t}, \ah{t}, \x{t+1})
%\end{eq}

%The optimal robot policy can be computed as

%\begin{eq}
%\policyropt = 
%\underset{\policyr}{\argmax}~
%\underset{\policyh}{\mathrm{E}}
%\vr{}(\x{0}|\policyr, \policyh)
%\label{eq:opt-robot-policy-exp}
%\end{eq}

%Key to solving~\equref{eq:opt-value-expect} is for
%the robot to have a model of human behaviors.

%One is to have a model over possible human policies.
%Given a model of human behaviors (possibly stochastic), 
%the other 
%challenge is to plan optimal robot actions.

%In this section, we first introduce a intention
%based human behavior model that explicitly depends human
%actions on their intentions and robot actions.
%The human behavior model is then embedded into a POMDP-lite
%for robot decision making.
%The resulting model, intention POMDP-lite, enables the
%robot to actively infer human intention for improved 
%performance.
%Although active exploration improves robot performance,
%it might be dangerous as well.
%To generate safe explorative actions, we further introduce 
%a safe 
%heuristic to explicitly guide robot explorations. 

\subsection{Intention-Driven Human Behavior Modeling}
\label{subsec:intention-human-behav}

Our insight in modeling human behaviors is that
people cannot be treated as obstacles that move, \ie, 
people select actions based on their \emph{intentions} and they
\emph{adapt} to what the robot does.

Following previous works on intention modeling~\cref{bandyopadhyay2013intention,sadigh2016information}, we 
assume that human intention can be represented as a single 
discrete random variable \intention{},
and we explicitly condition human actions on their intention, 
\ie, $\ah{t} \sim \policyh(\ah{t} | \x{t}, \intention{t})$.

% Bounded memory model
Apart from their own intention, people also adapt to the robot.
In the most general case, people condition their actions 
on the entire human-robot interaction history, \ie, 
\allowbreak $h_t = \{ \ar{0}, \ah{0}, \hdots, \ar{t-1}, \ah{t-1} \}$. 
However, the history $h_t$ may grow arbitrary long and make 
human actions extremely difficult to compute.
Even if we can compute it , 
this will still not be a good model for how people
make decisions in day to day tasks since
people are known to be not fully rational,
\ie, bounded rationality~\cref{rubinstein1998modeling,kahneman2003maps}.
In human-robot interaction, 
bounded rationality has been modeled by 
assuming that people have ``bounded memory'', and they
based their decisions only on most recent 
observations~\cref{nikolaidis2016formalizing}.

Our human behavior model connects the human intention model 
with the bounded memory model, 
\ie, people condition their actions on their intention and the last
$k$ steps of the history, $\history{k}{t} = \{\ar{t-k}, \allowbreak \ah{t-k}, \hdots, \ar{t-1}, \ah{t-1}\}$. 
Thus, the human behavior policy can be rewritten as
\begin{eq}
    \ah{t} \sim \policyh(\ah{t} | \x{t}, \history{k}{t}, \intention{t})
    \label{eq:human-policy}
\end{eq}

\subsection{Intention POMDP-lite}
\label{subsec:pomdp-lite}

% intention is partially observable, and static
A key challenge is that human intention can not be directly observed by
the robot, and therefore has to be inferred from human behaviors.
To achieve that,
We model human intention as a latent state variable in
a partially observable Markov decision process (POMDP).
In this paper,
we assume the human intention remains static during a single interaction.
Thus,
we can reduce the POMDP model to a POMDP-lite model~\cref{chen2016pomdp}, 
which can be solved more efficiently.

% intention-aware human behavior policy can be 
% integrated to a POMDP-lite model
To build the intention POMDP-lite model, 
we first create a factored state 
$s = (\x{},\intention{})$ that contains the fully
observable world state \x{} and the partially observable human intention \intention{}.
We maintain a belief $b$ over human
intention.
%and the belief is updated according
%to Bayes' rule after an observation $o$ has been received,
%\ie, $b' = \tau(b, \ar{}, o)$.
%For most human-robot interaction tasks, $o = x$.
The human behavior policy is 
embedded in the POMDP-lite transition dynamics, and
we describe in~\secref{sec:humanpolicy-safe} how we learn
it from data.
\figref{fig:intention-pomdp-lite} shows the 
intention POMDP-lite graphical model and
human-robot interaction flowchart.

\begin{figure}[t]
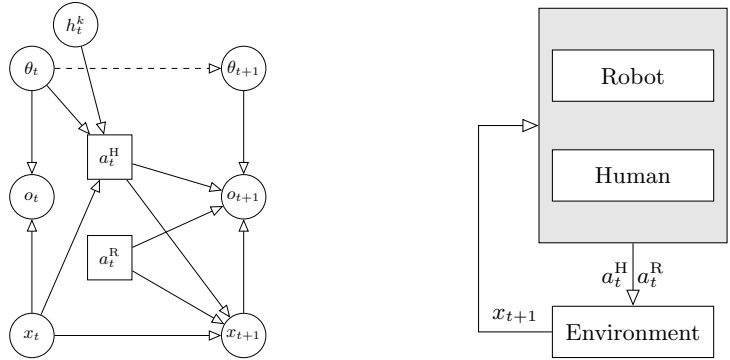

    \centering
    \includestandalone[width=.28\columnwidth]{figs/drive-pomdp-lite}
    \hspace{2.5cm}
    \includestandalone[width=.28\columnwidth]{figs/drive-flowchart}
    \caption{
        The intention POMDP-lite graphical model and the human-robot
        interaction flowchart.
        The robot action $a_t^\mathrm{\sss R}$ depends on the world state $\x{t}$
        and the belief 
        over human intention $\intention{t}$.
        The dashed arrow indicates the static assumption on
        the intention dynamics, \ie, $\intention{t} = \intention{t+1}$.
        } 
    \label{fig:intention-pomdp-lite}
\end{figure}

The solution to an intention POMDP-lite is a policy that maps
belief states to robot actions, \ie, $\ar{t} = \policyr(b_t, \x{t})$.
And it has two distinct objectives:
(1) maximizing rewards based on current information (exploitation);
(2) gathering information over human intention (exploration).

The Bayes-optimal robot trades off exploration/exploitation
by incorporating belief updates into its plans~\cref{duff2003design}.
It acts to maximize the following value function
\begin{eq}
    V^*(b_t, \x{t}) = 
    \underset{\ar{t}}{\max}
    \bigg{\{}
    \rwdR{}(b_t, \x{t}, \ar{t}) +
    \underset{\x{t+1}}{\sum}
    \gamma P(\x{t+1} | b_t, \x{t}, \ar{t}) V^*(b_{t+1}, \x{t+1})
    \bigg{\}}
    \label{eq:opt-bayesian-value}
\end{eq}
where 
$\rwdR{}(b_t, \x{t}, \ar{t}) = \underset{\intention{t}}{\sum} 
b_t(\intention{t}) P(\ah{t} | \intention{t}) \rwdR{}(\x{t}, \ar{t}, \ah{t})$ 
is the mean reward function, and 
$P(\x{t+1} | b_t, \x{t}, \ar{t}) = \underset{\intention{t}}{\sum} 
b_t(\intention{t}) P(\ah{t} | \intention{t}) 
P(\x{t+1} | \x{t}, \ar{t}, \ah{t})$
is the mean transition function,
where $P(\ah{t} | \intention{t})$
depends on the human behavior policy in~\equref{eq:human-policy}.
Note that $b_{t+1} = \tau(b_t, \ar{t}, \x{t+1})$ is the updated belief
after arriving at a new state $\x{t+1}$, where $\tau$ represents Bayes' rule.

However, Bayes-optimal planning is intractable in general.
An alternative approach to trade off exploration/exploitation
is the explicit modification of the reward function, \ie,
adding an extra reward bonus for exploration~\cref{chen2016pomdp}.
In this case, the robot acts to maximize
the following value function:
\begin{eq}
    \label{eq:value}
    \begin{split}
        \tilde{V}^*(b_t, \x{t}) = & \underset{\ar{t}}{\max} \bigg{\{} \disRB \rb(b_t, \x{t}, \ar{t}) +
    \rwdR{}(b_t, \x{t}, \ar{t}) + \\
    & \qquad \qquad \quad \underset{\x{t+1}}{\sum}
    \gamma P(\x{t+1} | b_t, \x{t}, \ar{t}) \tildeV^*(b_{t}, \x{t+1})
    \bigg{\}}
    \end{split}
\end{eq}
where $\rb(b_t, \x{t}, \ar{t})$ is the reward bonus term
that encourages the robot to explore.
\disRB is a constant scalar that explicitly trades off 
exploration and exploitation.
Note that the belief $b_t$ is not updated in this equation.
In other words, solving~\equref{eq:value} is 
equivalent to solving a mean MDP of current belief 
with an additional reward bonus, which 
is computationally efficient.
More importantly, the computed robot policy  
is near Bayes-optimal~\cref{chen2016pomdp},
given that the reward bonus is defined
as the expected $L_1$ distance between two consecutive
beliefs $\underset{b_{t+1}}{\ev} || b_{t+1} - b_t ||_1$.

%The computed robot policy from~\equref{eq:value} has two distinct objectives:
%\begin{itemize}
    %\item Exploitation. Maximize its own reward.
    %\item Exploration. Gather information over human's 
        %intention so that future actions are rewarded better.
%\end{itemize}

\subsection{Guided Exploration}
\label{subsec:safe-explore}

% Motivation of guided exploration 
The policy computed by~\equref{eq:value} enables
the robot to actively infer human intentions, however,
the human might not like to be probed in certain scenarios 
(\eg, \figref{fig:lane-switch-example}d).
Thus, guidance needs to be provided for more effective robot explorations.

We achieve guided exploration via human expert 
demonstrations. 
More specifically,
we maintain a probability distribution
over each state-action pair, 
where \allowbreak $\safef(\x{}, \ar{}) \in [0, 1]$ 
measures how likely the human expert will take action
\ar{} at state \x{}.
We will describe in~\secref{sec:humanpolicy-safe} how
we learn $\safef(\x{}, \ar{})$ from data.
The learned probability distribution is then embedded into the reward
bonus term as a prior knowledge, 
and our algorithm acts to maximize the following value function:
\begin{eq}
    \label{eq:value-safe}
    \begin{split}
        \safeVStar(b_t, \x{t}) = & \underset{\ar{t}}{\max} \bigg{\{} \disRB \safef(\x{t}, \ar{t}) \rb(b_t, \x{t}, \ar{t}) +
    \rwdR{}(b_t, \x{t}, \ar{t}) + \\
    & \qquad \qquad \quad \underset{\x{t+1}}{\sum}
    \gamma P(\x{t+1} | b_t, \x{t}, \ar{t}) \safeVStar(b_{t}, \x{t+1})
    \bigg{\}}
    \end{split}
\end{eq}
Compared with~\equref{eq:value}, the reward bonus term is 
multiplied by $\safef(\x{t}, \ar{t})$,
which discourages robot exploration when $\safef(\x{t}, \ar{t})$ is small.
Although we discourage the robot to explore certain state
action pairs,
the theoretic results in the POMDP-lite paper~\cref{chen2016pomdp} retains, 
\ie, the robot policy remains near Bayes-optimal.
%Due to space constraints,
The proof of the theorem is deferred to the appendix.
%available in
%the full versoin of the paper ().
\begin{theorem}    \label{thm:theorem1}
    Let $\optply_t$ denote the policy followed by our algorithm 
    at time step $t$. Let $\x{t}$, $b_t$ denote the state and belief
    at time step $t$. 
    Let $|\X{}|$, $|\AR{}|$ denote the size of the state
    space and robot action space.
    Let \safeConstant denote the minimal
    value of the probability distribution \safef(\x{}, \ar{}), \ie,
    $\safeConstant = \min \{ \safef(\x{}, \ar{}) \}$.
    Suppose $\safeConstant > 0$, $\disRB \allowbreak = \allowbreak O \allowbreak (\frac{|\X{}|^2, |\AR{}|}{\safeConstant(1-\gamma)^2})$,
    for any inputs 
    $\forall \delta > 0$, $\epsilon > 0$,
    with probability at least $1-\delta$,  
    \begin{dm}
        V^{\optply_t}(b_t, \x{t}) \geq V^*(b_t,\x{t}) - 4 \epsilon.
    \end{dm}
    %a small number of time steps.
\end{theorem}
In other words, our algorithm is $4 \epsilon$-close to the Bayes-optimal
policy, for all but
$m = O\big{(}\mathrm{poly}(|\X{}|, \allowbreak |\AR{}|, \frac{1}{\epsilon},
\frac{1}{\delta}, \frac{1}{1-\gamma})\big{)}$ time steps.

\section{Learning Human Behavior Policies and Guided Robot Exploration}
\label{sec:humanpolicy-safe}

Nested within the intention POMDP-lite model is
the human behavior policy \allowbreak $\policyh(\ah{t} | \x{t}, \history{k}{t}, \intention{t})$,
and the guided exploration distribution \allowbreak $\safef(\x{t}, \ar{t})$. 
We adopt a data-driven approach and learn those two
models from data for the interactive driving tasks.
Note that suitable probabilistic models derived from alternative
approaches can be substituted for these learned models.

\subsection{Data Collection}
\label{subsec:data-collection}

% Setups
\begin{figure}[t]
    \centering
    \includegraphics[width=.99\columnwidth]{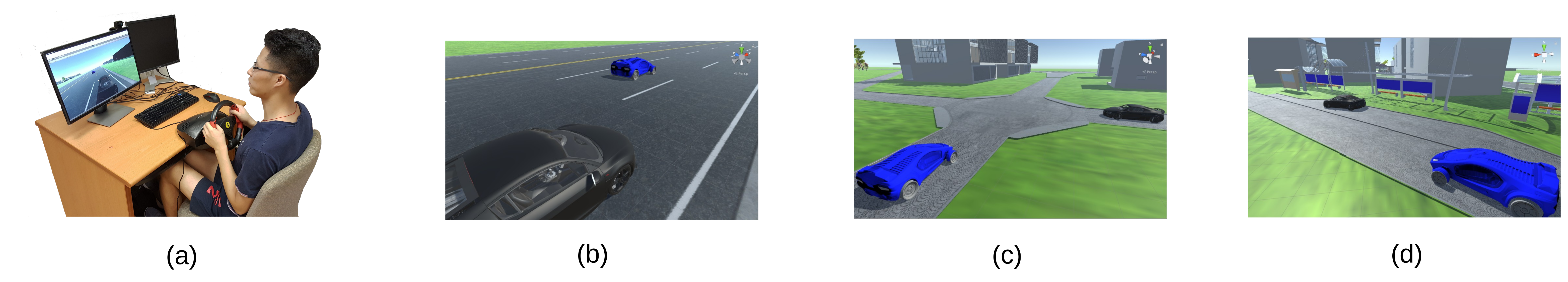} 
    \caption{Simulation setup.
        (a)    A human driver interacts
        with a robot car in the driving simulator powered
        by the Unity 3D engine. (b,c,d)
        Three interactive driving tasks: lane-switch, intersection,
        and lane-merge.
    }
    \label{fig:experiment-setup}
\end{figure}

\paragraph{Interactive driving tasks.}

\figref{fig:experiment-setup} shows the simulation setup
and three typical interactive driving tasks.
% ~\footnote{
%     \figref{fig:experiment-setup} shows the snapshots
%     from the Unity 3D engine (\url{https://unity3d.com}).
%     The remaining simulation images in this paper are
%     simplified for clearer annotations.}.
For simplicity,
we assume both the human-driven car and the robot car 
follow a fixed path respectively.
For example, in the intersection scenario, each of them will
follow a path that goes straight forward.
With this assumption, the human driver and the robot 
only need to control the speed of the vehicles. The
steering angle is controlled by a 
path tracking algorithm~\cref{coulter1992implementation}.
One exception is the lane-switch scenario, where the robot 
can decide when to switch lanes, 
and has two additional actions, \ie, $\{ switch-left, switch-right \}$.
Once the robot decides to switch lanes, a new path will
be generated online and the path tracking algorithm will 
start following the new path.

In general, human intention in driving scenarios has multiple
dimensions, \eg, turn-left/turn-right, aggressive/conservative,
\etc.
In this paper, we focus on the last dimension, \ie, the
human driver can be either aggressive or conservative. 

\paragraph{Participants.}
We recruited $10$ participants (3 female, 7 male) in
the age range of $22-30$ years old.
All participants owned a valid driving license.

\paragraph{Design.}
The human-driven car was controlled by one of the participants,
and the robot car was controlled by a human expert.
Note that the human expert was not recruited from the general 
public but one of the experiment conductors.
Intuitively, one can treat the participants as the human drivers 
that the robot will interact with, 
and the human expert as the owner of robot car whom teaches the 
robot how to act in different scenarios.

We learn a human drivers' behavior model from the controls
recorded from the participants.
To capture different human intentions, we asked each participant
to perform the task as an aggressive human driver and as a 
conservative human driver.
Note that the notion of ``aggressive/conservative'' is 
subjective and may vary
among different participants. Thus, we expect certain
amount of variance in our human behavior prediction model.

Similarly, we learn a guided exploration distribution from
the controls recorded from the human expert.
Since safety is our primary concern in the driving tasks,
the human expert was told to drive carefully.

\paragraph{Procedure.}
Before the simulation started, the participant was asked
to follow one of the driver intentions, \ie, aggressive or 
conservative.
Once the simulation started, the participant and the 
human expert could control the speed of their vehicles
respectively via an accelerator pedal that provides
continuous input.
In the lane-switch scenario, the human expert could
also decide when to switch to the other lane.
The simulation ends once the robot car has achieved its goal,
\ie, crossed the intersection or switched to another lane.

\paragraph{Data format.}
\begin{wrapfigure}{r}{0.4\columnwidth}
  \vspace*{-30pt}
    \centering
    \includegraphics[width=.35\columnwidth]{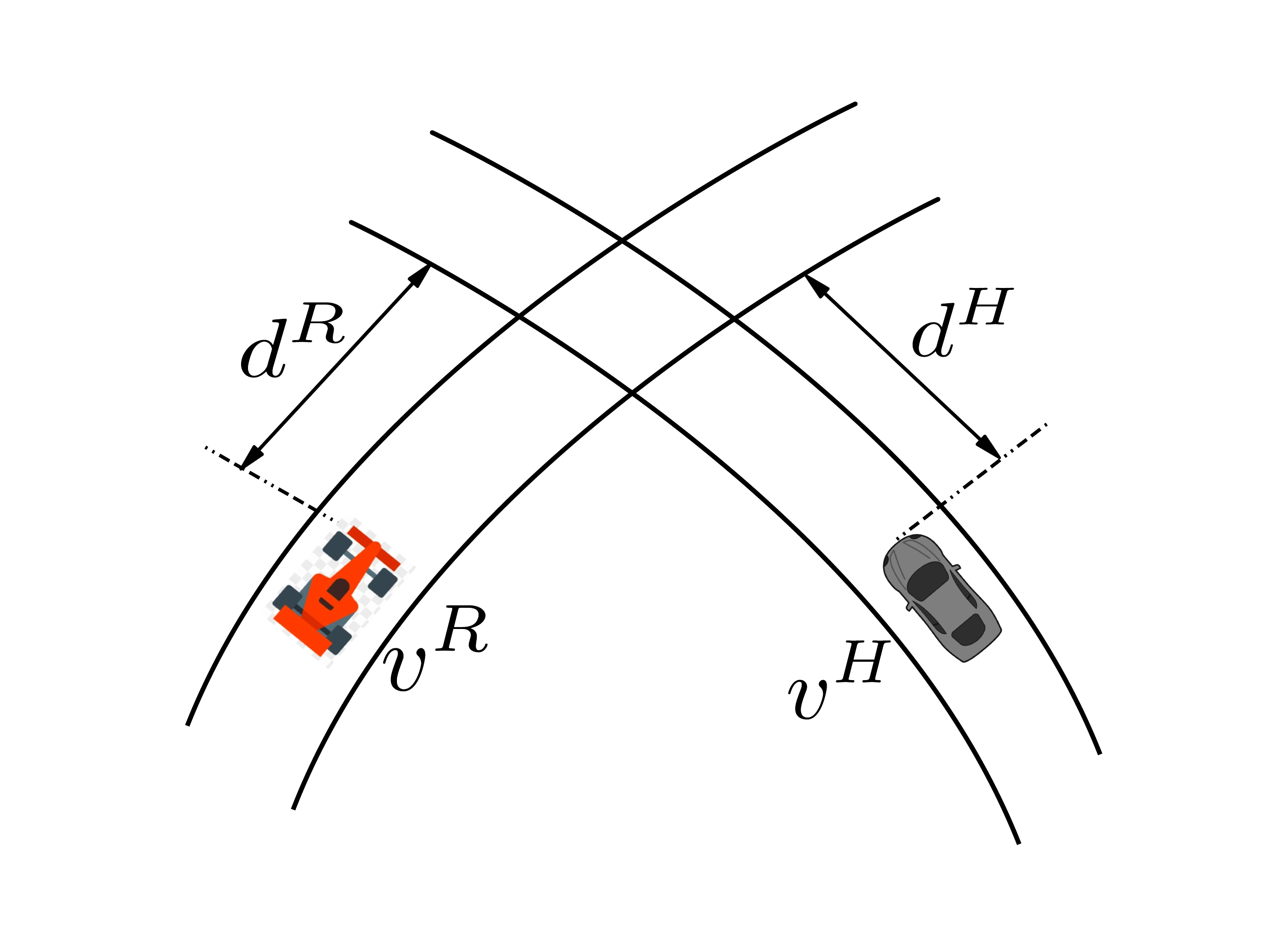}
    \caption{The world state for an interactive driving task.}
    \label{fig:world-state-driving}
\end{wrapfigure}
 The world state in the driving task is depicted 
in~\figref{fig:world-state-driving}, $x = \{ \dH{}, \dR{}, \vh{}, \vr{} \}$,
where \dH{}, \dR{} are the vehicles' distance to the 
potential colliding point, and \vh{}, \vr{} are the vehicles'
current speed.

For each simulation sequence, we recorded two set of data,
$\hData_i^H$ and $\hData_i^G$, for learning the human behavior model
and guided exploration distribution respectively.
The data recorded can be written as
follows:
\begin{align*}
    \hData_i^H = \intention{i} \cup \{\tuple{\x{0}, \ar{0}, \ah{0}},
        \hdots,
        \tuple{\x{K_i}, \ar{K_i}, \ah{K_i}} \}, \ 
    %&\textit{and,} \\
    \hData_i^G = \{ \tuple{\x{0}, \ar{0}}, \hdots, \tuple{\x{K_i}, \ar{K_i}}\}
\end{align*}
where 
$\intention{i}$ is the human intention at the $ith$
interaction,
$K_i$ is the number of steps in the $ith$ interaction.
\x{t} is the world state at time step $t$.
$\spdR{t}, \ar{t}$ are the speed and acceleration of the robot
car at time step $t$.
$\spdH{t}, \ah{t}$ are the speed and acceleration of the human-driven car at time step $t$.

% Assign task
Each participant was asked to perform the driving task
$8$ times as an aggressive driver and $8$ times as a 
conservative driver.
We have $10$ participants in total, 
this gives us $160$ sequences of interactions for the 
learning purpose, \ie, $80$ for the aggressive human driver,
$80$ for the conservative human driver, and $160$ for the 
guided exploration distribution.

\subsection{Human Behavior Policy}
\label{subsec:human-policies}

A human behavior policy is 
a function that maps the current world state
\x{t}, bounded history \history{k}{t} and human intention
\intention{t} to human actions (\equref{eq:human-policy}),
where human actions are accelerations in our driving tasks.
The Gaussian Process (GP) places a distribution over 
functions. It serves a non-parametric form of
interpolation, and it is extremely robust to unaligned
noisy measurements~\cref{rasmussen2004gaussian}.
In this paper, we use GP as the human behavior prediction model.

We learn a GP model for each human driver type, \ie,
aggressive or conservative.
Our GP model is
specified by a set of mean and covariance functions,
whose input includes the world state $\x{t}$ and
bounded history $\history{k}{t}$, \ie,
$\gx{} = \tuple{\x{t}, \history{k}{t}}$.
To fix the dimension of the input, 
we add paddings with value of $0$ to the history if $t < k$.

\paragraph{Mean function.}
The mean function is a convenient way of incorporating
prior knowledge.
We denote the mean function as 
$\gmean(\gx{})$, and set it initially to be the mean of 
the training data.
This encodes the prior knowledge that, without any additional
knowledge, we expect the human driver to behave similarly to
the average of what we have observed before.

\paragraph{Covariance function.}
The covariance function describes the correlations
between two inputs, \gx{} and $\gx{}'$,
and we use the standard radial-basis function 
kernel~\cref{vert2004primer}.

\vspace{-1.5em}
\begin{eq}
    \begin{split}
        \gcov(\gx{}, \gx{}') = & \exp \bigg{(} -1/2 \bigg{(}
        \big{(} \frac{\dH{} - \dHn{}}{\glen{d}} \big{)}^2 +
        \big{(} \frac{\dR{} - \dRn{}}{\glen{d}} \big{)}^2 +
        \big{(} \frac{\spdH{} - \spdHn{}}{\glen{v}} \big{)}^2 + \\
        & \quad \quad \quad \big{(} \frac{\spdR{} - \spdRn{}}{\glen{v}} \big{)}^2 +
        \underset{\ar{}, \ah{} \in \history{k}{t}}{\sum}
        \bigg{(}
        \big{(} \frac{\ar{} - \arn{}}{\glen{a}} \big{)}^2 +
        \big{(} \frac{\ah{} - \ahn{}}{\glen{a}} \big{)}^2
        \bigg{)} \bigg{)}
    \end{split}
\end{eq}
where the exponential term encodes that similar 
state and history should make similar
predictions.
The length-scale \glen{d}, \glen{v} and \glen{a}
normalize the scale of the data, \ie, distance,
speed and acceleration.

\paragraph{Training.}
We train the GP model with the scikit-learn package~\cref{scikit-learn},
where $80\%$ data is used for training and $20\%$ data is
used for testing.
We evaluate different values of $k$.
\figref{fig:GP-MSE} shows the mean squared train/test error
with respect to $k$.
The errors are large when $k=0$ (people ignore the robot), but converge
quickly within $2$ steps.
This supports our intuition that people adapt to the robot
but have a bounded memory. 
In the remainder of the paper, GP models with
$k=2$ are used to predict human actions.

\begin{figure}[t]
    \centering
    \includegraphics[width=.3\columnwidth]{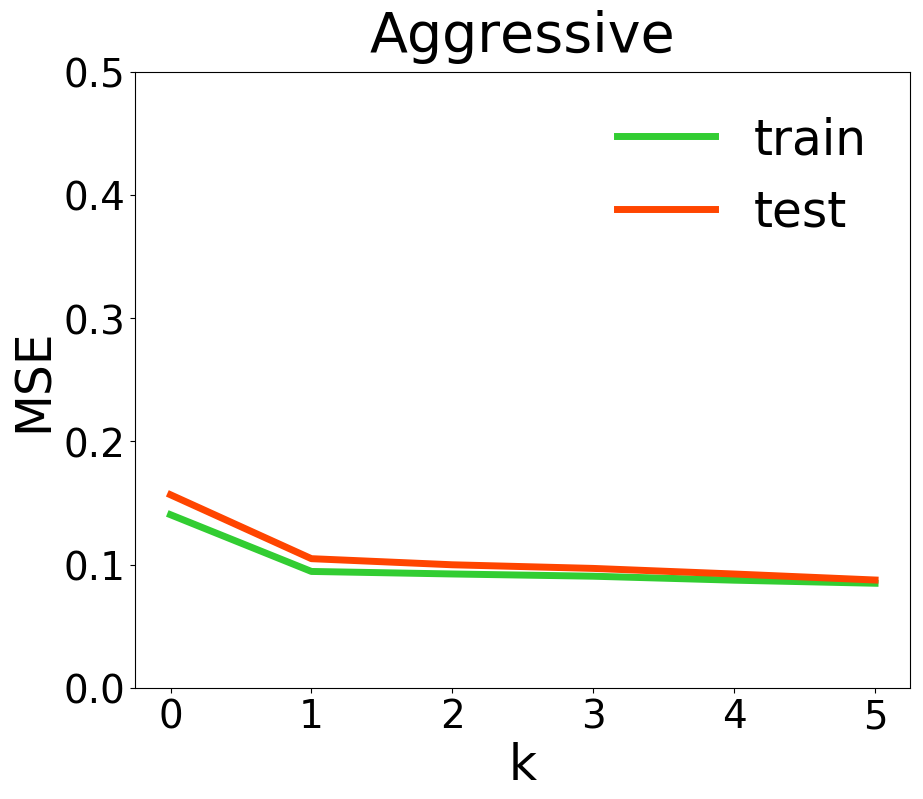}
    \hspace{2.0cm}
    \includegraphics[width=.3\columnwidth]{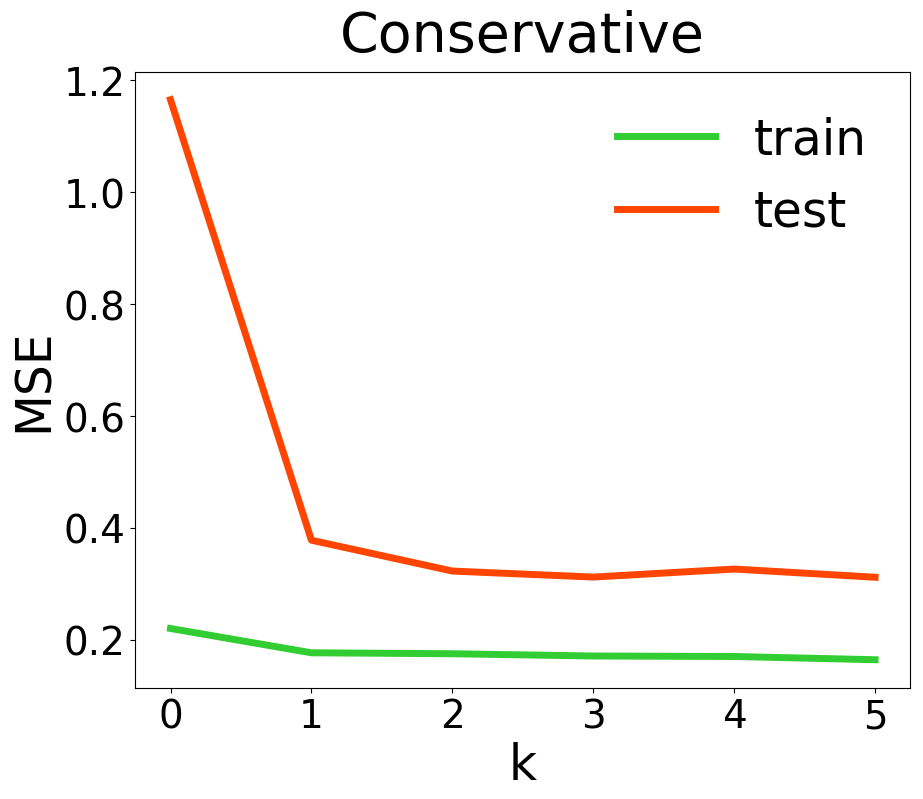}
    \caption{
        The Mean Squared Error (MSE) of GP predictions  with
        increasing history length $k$. 
        The MSE stabilizes for $k\geq 2$, confirming 
         that humans have short, bounded memory in driving tasks.
        %Aggressive human driver (left), and conservative human
        %driver (right).
    }
    \label{fig:GP-MSE}
\end{figure}

\figref{fig:hpolicy} shows GP predictions 
for some example scenarios. 
The plots are generated by varying one of the input 
variables, \ie, $d$ or $v$, while fixing the others.
Due to space constraints, we only
show selected plots where the aggressive human driver
and the conservative human driver are most distinguishable.
The key message in those plots is: the conservative human
driver slows down if there is a potential collision in
the near future, while the aggressive human driver keeps
going and ignores the danger.
The robot can take advantage of this difference to 
actively infer human intentions.

\begin{figure}[t]
    \centering
    \includegraphics[width=.99\columnwidth]{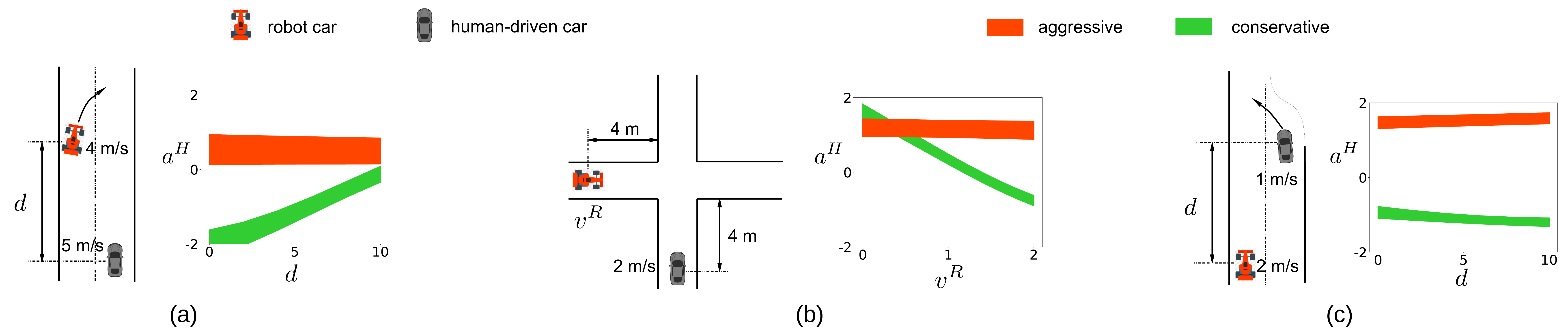}
    \caption{GP predictions of human accelerations, 
        mean and standard error ($y$ axis),
        for some example driving scenarios.
        (a) Predicted human acceleration 
        \wrt the distance
        between two vehicles (lane-switch).
        (b) Predicted human acceleration
        \wrt the velocity of the robot car 
        (intersection). 
        (c) Predict human acceleration 
        \wrt the distance between two vehicles 
        (lane-merge).
    }
    \label{fig:hpolicy}
\end{figure}

\paragraph{Human behavior policy table.}
The human behavior policy will be queried frequently
during the robot planning phase, and GP prediction 
is too slow for online POMDP planning.
Instead, we build a policy table offline for each human
driver type, where we store all the GP predictions.
During the online phase, the POMDP planning algorithm only
needs to query the table and it is much faster.

To build a policy table, 
we discretize the state space and action
space as follows:
\begin{itemize}
    \item Distance ($\mathrm{m}$): near $[0, 5)$, middle
        $[5, 20)$, far $[20, +\infty)$.
    \item Speed ($\mathrm{m/s}$): low $[0, 1)$, middle
        $[1, 5)$, high $[5, +\infty)$.
    \item Acceleration ($\mathrm{m/s^2}$):  decelerate $(-\infty, -0.2)$, keep $[-0.2, 0.2]$, accelerate \allowbreak $(0.2, \allowbreak +\infty)$.
\end{itemize}
With this level of discretization,
the policy table will have $3^8$ entries when history
length $k=2$.

\subsection{Probability Distribution for Guided Exploration}
\label{subsec:probability-guided}

For the autonomous driving task, we are mostly concerned
with a state-action pair being safe or unsafe. 
Consequently, the probability $\safef(\x{}, \ar{})$ 
measures how safe
it is for the robot to explore \tuple{\x{}, \ar{}},
\ie, safe probability.

\paragraph{Prior.}
Since a state-action pair is either safe or unsafe, 
a natural means is to 
use Beta distribution as a prior , \ie,
$\betaD(\alpha(\x{}, \ar{}), \beta(\x{}, \ar{}))$.
The initial safe probability can be computed as

\begin{dm}
    \safef(\x{}, \ar{}) = \frac{\alpha\tuple{\x{}, \ar{}}}{\alpha\tuple{\x{}, \ar{}} + \beta\tuple{\x{}, \ar{}}}
\end{dm}

Initially, we set $\alpha\tuple{\x{}, \ar{}} = 0.05, \beta\tuple{\x{}, \ar{}} = 5, \forall \x{} \in \X{}, \ar{} \in \AR{}$.
This implies that all state-action pairs are close to unsafe 
without seeing any human demonstrations.

\paragraph{Posterior.}
Given a set of human demonstration data
$\hData^G = \{\tuple{\x{0}, \ar{0}}, \hdots, \allowbreak \tuple{\x{N}, \ar{N}}\}$,
the posterior of the safe
probability can be computed as 

\begin{dm}
    \safefp(\x{}, \ar{}) = \frac{\alpha\tuple{\x{}, \ar{}} + n(\x{}, \ar{})}{\alpha\tuple{\x{}, \ar{}} + \beta\tuple{\x{}, \ar{}} + n(\x{}, \ar{})}
\end{dm}
where $n(\x{}, \ar{})$ is the number of times that
\tuple{\x{}, \ar{}} appeared in the human demonstration data.

\paragraph{Safe probability table.}
Similar to the human behavior policy, we store all the
safe probabilities in a table, and we follow the same
discretization intervals for the human behavior policy table.

\section{Simulation Experiments}
\label{sec:experiments}

Our approach, intention POMDP-lite with Guided exploration
(\ipomg),
enables the robot to 
actively infer human intentions and 
choose
explorative actions that are similar to the human expert. 
In this section,
we present some simulation results on several interactive
driving tasks, where the robot 
car interacts with a simulated human driver.
We sought to answer the following two questions:
\begin{itemize}
    \item \textbf{Question A.} Does active
        exploration improve robot efficiency? 
    \item \textbf{Question B.} Does guided exploration 
        improve robot safety?
        %generate safer robot explorative actions in the
        %driving tasks?
\end{itemize}

\paragraph{Comparison.}
To answer question A, we compared 
\ipomg with a myopic robot policy that does not
actively explore, \ie, the reward bonus
in~\equref{eq:value} was set to be $0$.
To answer question B, we compared 
\ipomg with the original intention POMDP-lite model 
(\ipom) without guided exploration.

Since the driving scenarios considered in this paper
are relatively simple, one may argue that a simple
heuristic policy might work just as well.
To show that is not the case, we designed
an additional baseline, \ie, heuristic-$k$, and 
it works as follows:
the robot explores at 
the first $k$ steps, \eg, accelerate or switch lane,
then the robot
proceeds to go if the human slows down,
%(ii) there is no collision when the human and the robot continue their 
%current velocity.
otherwise, the robot waits until the human has crossed.

\paragraph{Parameter settings.}
Both \ipom and \ipomg 
need to set a constant scalar $\disRB$ that
trades off exploration and exploitation (\equref{eq:value} and
\equref{eq:value-safe}).
Similar to previous works~\cref{kolter2009near}, we evaluated a
wide range of values for \disRB, and chose the one 
that had the best performance on \ipom.
In this way, \disRB favors \ipom more than \ipomg.

For the heuristic policy, we need to choose 
parameter $k$, \ie, the number of steps that robot
explores at the beginning.
We evaluated heuristic policies with different 
values of $k \in \{ 1, 2, 3, 4 \}$,
and chose the one that had the best performance.

\paragraph{Performance measures.}
To measure the efficiency of a policy, we used the time
taken for the robot to achieve its goal as the 
performance measure, \ie, $T(goal)$, the less the better.

To measure the safety of a policy, we used the near-miss
rate as the performance measure, \ie, $P(near-miss)$,
since we didn't observe any accidents in the simulations.
We adopted the definition of near-miss from a seminal
paper in the field of traffic safety control~\cref{hayward1972near},
where near-miss was defined based on the time-measured-to-collision (TMTC).
TMTC is the time required for the two vehicles to collide
if they continue at their current velocity.
Intuitively, TMTC is a measurement of danger, and 
lower TMTC value implies that the scenario is more dangerous.
According to~\cref{hayward1972near}, 
a near-miss happens if the value of TMTC is lower than 
$\mathbf{1}$ second,
and an analysis over the films taken with the surveillance
system at an urban interaction suggested a near-miss rate
of $\mathbf{0.35}$ in daily traffic.

\paragraph{Simulation setup.}
For all the simulations performed, 
the robot car was controlled by one of the algorithms 
above.
The human-driven car was controlled by learned human 
behavior policy in~\secref{subsec:human-policies}.
For each simulation run, 
the human driver was set to be aggressive or conservative 
with $0.5$ probability.

The state space \x{} was set to be continuous.
The robot action space was set to be discrete, \ie, 
\{Accelerate, Keep, Decelerate\},
which controls the speed of the car.
The steering angle of the car was controlled independently
by a path tracking algorithm~\cref{coulter1992implementation}.
All planning algorithms were given $0.33$ second per step 
to compute the robot actions online.

\subsection{Driving Scenarios}
\label{subsec:driving-scenarios}

%\begin{figure}[t]
    %\centering
    %\includegraphics[width=.99\columnwidth]{figs/policy-examples/driving-scenarios-safe.pdf}
    %\caption{
        %Scenarios A1-A3,
        %and a comparison of the active exploration policy
        %vs. the myopic policy.
    %}
    %\label{fig:driving-scenarios-safe}
%\end{figure}

%\begin{figure}[t]
    %\centering
    %\includegraphics[width=.99\columnwidth]{figs/policy-examples/driving-scenarios-unsafe.pdf}
    %\caption{
        %Scenarios B1-B3,
        %and a comparison of the guided exploration 
        %policy vs. 
        %the policy without guided exploration.
    %}
    %\label{fig:driving-scenarios-unsafe}
%\end{figure}

\begin{figure}[t]
    \centering
    \includegraphics[width=.99\columnwidth]{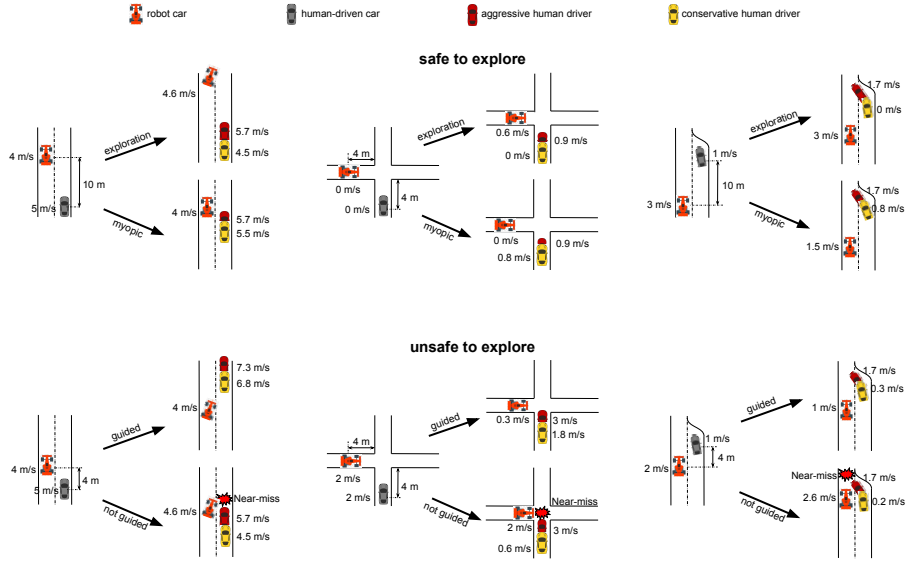}
    \caption{
        Top row: comparison of the active exploration policy and the myopic policy, 
        when exploration is safe. 
        Bottom row: comparison of policies with and without guided exploration, 
        when exploration is unsafe.
    }
    \label{fig:driving-scenarios}
\end{figure}

To test the effectiveness of active exploration, we
selected three interactive driving scenarios
where robot exploration is safe 
(\figref{fig:driving-scenarios}, top row).
Initially, the robot car either had a safe distance to
the human-driven car or had low speed.
This gave the robot car enough space/time to explore 
without too much danger of colliding to the human-driven car.
The myopic robot waited until the 
human-driven car had crossed or slowed down.
However, the active exploration robot 
started to switch to the other lane or accelerate
to test if the human was willing to yield.
If the human was conservative and slowed down, 
the robot proceeded to go and improved its efficiency.

To test the effectiveness of guided exploration,
we selected another three interactive 
driving scenarios where robot exploration is unsafe 
(\figref{fig:driving-scenarios}, bottom row).
Initially, the robot car was near to the human-driven car 
and it had high speed.
Without guided exploration,
the \ipom robot chose to switch lane or cross the 
intersection,
which might cause near-misses.
%since the time left for the human to react was short.
On the other hand, with guided exploration,
the \ipomg robot chose to not explore and waited for 
the human to cross first.

\subsection{Quantitative Results}
\label{subsec:quantitative-results}

\begin{figure}[t]
    \centering
    \includegraphics[width=.98\columnwidth]{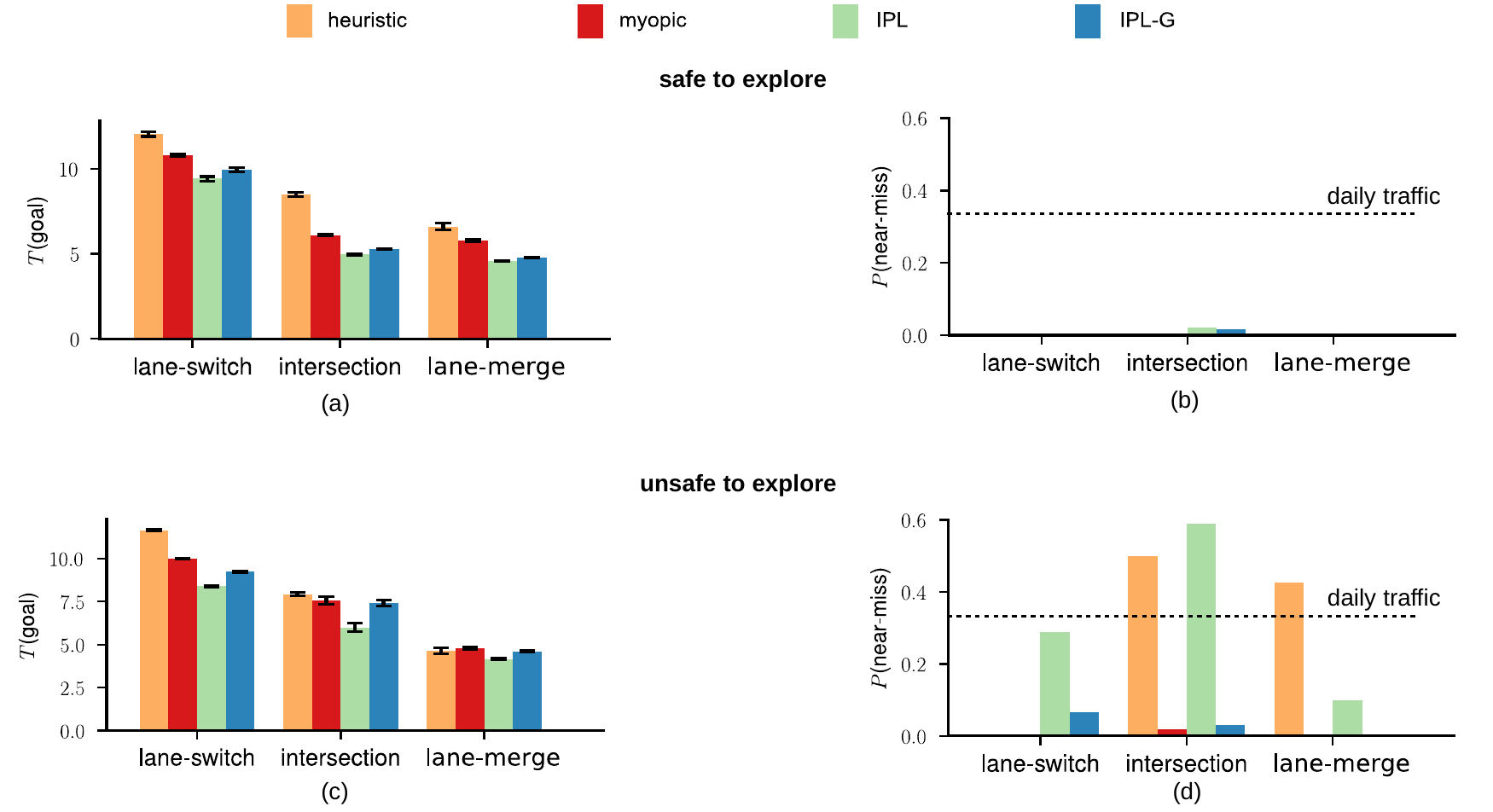}
    \caption{Performance results.
        Top row:
        when it was safe to explore,
        the active exploration policy (\ipom and \ipomg) 
        achieved better efficiency with nearly zero near-misses.
        Bottom row:
        when it was unsafe to explore,
        the \ipomg robot
        achieved significantly lower near-miss rate
        compared with the robot that explores without
        guided exploration (\ipom).
        }
    \label{fig:performance-comparison}
\end{figure}

We performed $200$ simulation runs for each scenario.
The performance results are shown in \figref{fig:performance-comparison}.

When robot exploration was safe,
both \ipom and \ipomg actively
explored. This significantly reduced the time
taken for the robot to achieve its goal,
compared with the robot that followed a myopic policy
or a heuristic policy (\figref{fig:performance-comparison}a).
Very few near-misses ($<0.02$) have been observed for all the
robot policies, which supports our intuition that 
it is safe for the robot to explore in those scenarios
(\figref{fig:performance-comparison}b).
Note that the dashed line represents the near-miss rate in 
daily traffic, which is adopted from a seminal paper in
traffic safety control~\cref{hayward1972near}.
It helps us to calibrate the safety levels of 
different algorithms.

When robot exploration was unsafe,
The \ipom robot was 
aggressive at gathering information. 
This indeed made the robot car more efficient 
(\figref{fig:performance-comparison}c).
However, this is counterbalanced by the fact that it
also led to a lot of near-misses 
(\figref{fig:performance-comparison}d).
The \ipom robot
incurred more near-misses than the daily traffic in the intersection scenario, and it
had significantly higher near-miss rate than the 
\ipomg robot across all tasks, which supports our
intuition that guided exploration can significantly
improve robot safety in driving tasks.
%In driving tasks, safety should be valued more than efficiency.
The heuristic policy explored at the first $2$ steps,
which already caused a lot of near-misses.
This implies that arbitrary exploration is dangerous
in general, and should not be encouraged.

% Performance measure
%\begin{figure}[t]
    %\centering
    %\begin{subfigure}{0.9\columnwidth}
        %\centering
        %\includegraphics[width=.8\columnwidth]{figs/exp-plots/performance-legend.pdf}
    %\end{subfigure}
    %\begin{subfigure}{0.9\columnwidth}
        %\centering
        %\includegraphics[width=.4\columnwidth]{figs/exp-plots/safe-exploretime2cross.pdf}
        %\hspace{1.5cm}
        %%\includegraphics[width=.4\columnwidth]{figs/exp-plots/safe-explorenearmisses-rate.pdf}
        %\includegraphics[width=.4\columnwidth]{figs/exp-plots/near-miss-rate-safe.pdf}
    %\end{subfigure}
    %\caption{
        %Performance results in scenarios where robot exploration
        %is safe.
        %The active exploration robot achieved better
        %efficiency with nearly zero near-misses.
        %}
    %\label{fig:active-information-stats}
%\end{figure}

%% Performance measure
%\begin{figure}[t]
    %\centering
    %\begin{subfigure}{0.9\columnwidth}
        %\centering
        %\includegraphics[width=.8\columnwidth]{figs/exp-plots/performance-legend.pdf}
    %\end{subfigure}
    %\begin{subfigure}{0.9\columnwidth}
        %\centering
        %\includegraphics[width=.4\columnwidth]{figs/exp-plots/unsafe-exploretime2cross.pdf}
        %\hspace{1.5cm}
        %%\includegraphics[width=.4\columnwidth]{figs/exp-plots/unsafe-explorenearmisses-rate.pdf}
        %\includegraphics[width=.4\columnwidth]{figs/exp-plots/near-miss-rate-unsafe.pdf}
    %\end{subfigure}
    %\caption{
        %Performance results in scenarios where robot exploration
        %is unsafe.
        %The \ipomg robot
        %achieved significantly lower near-miss rate
        %compared with the robot that explores without
        %guided exploration.
    %}
    %\label{fig:safe-explore-stats}
%\end{figure}

% Analysis of sampled trajectories
We use the averaged robot
trajectories from the lane-switch task as an
example to illustrate the robot policies from different algorithms.
(\figref{fig:sampled-trajectories}).
When it was safe to explore (\figref{fig:sampled-trajectories},
top row), the \ipom robot and the \ipomg robot were able to identify human
intentions quickly, \ie, the belief $P$(conservative) converged to $1$ or $0$
quickly.  Consequently, if the human was conservative, the robot proceeded to
go and improved its efficiency, \ie, $D$(goal) decreased faster than the
myopic robot.  When it was unsafe to explore
(\figref{fig:sampled-trajectories}, bottom row), guided exploration prevented
the \ipomg robot to explore, \ie, its belief $P$(conservative) converged
slower than the belief of the \ipom robot (without guided exploration).
Consequently, the \ipomg robot was less aggressive and it encountered
significantly less near-misses compared to the \ipom robot
(\figref{fig:performance-comparison}d).

% Sampled trajectories
\begin{figure}[t]
    \centering
    \includegraphics[width=.98\columnwidth]{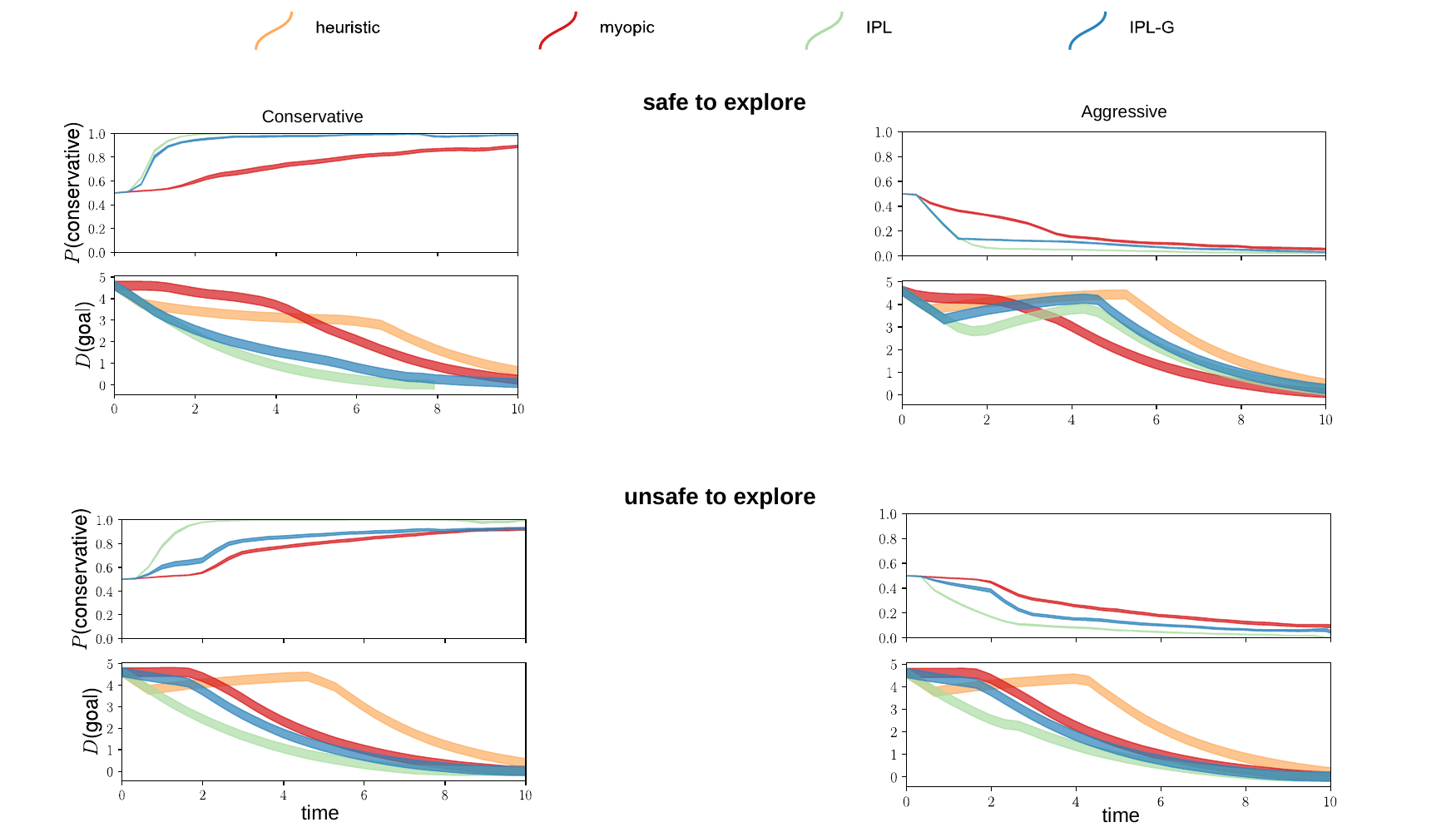}
    \caption{
        Averaged robot trajectories in the lane-switch scenario.
        The plots show the mean and standard error
        of the belief $P$(conservative) and the 
        remaining distance to the robot goal $D$(goal),
        with respect to time.
        }
        \label{fig:sampled-trajectories}
\end{figure}

% Discussion
%
% -Summary
%
% -Limitations:
%   1. model assumption: static intention
%   2. Only two types of human intentions in the experiment
%   4. performance measure
%
% -Generalization of our approach 
%

\section{Discussion}
\label{sec:discussion}

Intention POMDP-lite embeds a learned human behavior model in a probabilistic
decision framework. It is task-driven and gathers information on human
intentions only when necessary to improve task performance. 
% different from previous approaches that focus on identifying intentions
% (\eg,~\cref{sadigh2016information}).
The simulation experiments suggest that active exploration enables the robot
car to infer human intentions more effectively and improve driving
performance, compared with the myopic policy.
% Furthermore, embedding the intention-based human behavior
% model into a decision framework makes our approach task
% driven, \ie, it gathers information over intentions
% only when necessary, which is 
By leveraging expert demonstration data, guided exploration biases exploration
in promising directions and prevents exploration when it is ineffective or
unsafe.  Overall,  intention POMDP-lite with guided exploration
tries to answer the question of \emph{how} and \emph{when} to perform active
exploration in human-robot interactive tasks.
% Unlike handcrafted heurtistic driving policies,
The  diverse robot  behaviors emerge
automatically from the decision framework, without the  kind of explicit manual
programming required for heuristic driving policies. 

Although our experiments are specific to autonomous driving, understanding
human intentions is essential in a wide range of human-robot interaction
tasks, and our overall approach is generally applicable.
For example,  a kitchen assistant robot tries to understand the recipe that a
human tries to follow and helps by preparing the ingredients.
Similarly, a robot tries to understand a human's plan for
assembling a piece of furniture and gathers the parts and tools required.
In these settings, the risk of exploration may not be safety, but potentially
negative impact on task performance.
% cook dinner, but it is unsure what recipe the human is following. The
% collaborative cooking task can be modeled as an intention POMDP-lite, where
% the unknown intention is the recipe.  In addition, the guided exploration can
% be customized for robot operation in the kitchen settings, where safety might
% not be the first priority (e.g., the robot is not harmful).

%% overall approach
The current work has several limitations that require further
investigation. One important issue is human intention modeling.  We treat
intention as a single discrete random variable.  Specifically, our current
experiment design considers only two intentions: aggressive or conservative.
The simplified intention model allows us to analyze the core technical issues
without the interference of confounding factors. In reality, human drivers may
exhibit a mixture of aggressive and conservative behaviors.
% This can be captured by a richer representation of intentions.
A multi-dimensional, continuous parameterization would provide a richer and
more accurate intention representation.  In addition, we assume that intention
is static, while it may change over time.  We expect our approach to be
robust against intention changes, as online planning handles
unexpected changes naturally, but we plan to conduct a human subject study to
examine this issue further.

\medskip
\begin{small}
    \noindent\textbf{Acknowledgments.}
    This work is supported in part by the Singapore NRF through the SMART
    Future Urban Mobility IRG and MOE AcRF grant  MOE2016-T2-2-068. 

    \bibliographystyle{styles/bibtex/splncs03}
    %% cram the bibliography
    %\renewcommand{\baselinestretch}{0.9}
    \bibliography{wafr2018}
\end{small}

%\newpage
\input{wafr18-appendix}

\end{document}

%% file: wafr18-macros.tex
\newcommand{\ie}{i.e.}
\newcommand{\wrt}{w.r.t \xspace}
\newcommand{\etc}{e.t.c}
\newcommand{\etal}{et al.}
\newcommand{\eg}{e.g.}

\newcommand{\tuple}[1]{\ensuremath{( #1 )}\xspace}
\newcommand{\ev}{\ensuremath{\mathrm{E}}\xspace}

\graphicspath{{figs/}}
    \DeclareGraphicsExtensions{.pdf,.jpg,.png,.eps}

\newcommand{\cref}[1]{\cite{#1}}
\newcommand{\figref}[1]{Fig.~\ref{#1}}

\newcommand{\equref}[1]{Equation~\ref{#1}}
\newcommand{\thmref}[1]{Theorem~\ref{#1}}
\newcommand{\lemref}[1]{Lemma~\ref{#1}}

\newcommand{\secref}[1]{Section~\ref{#1}}

\usepackage{xspace}
\usepackage{tikz}
\usetikzlibrary{decorations.markings}

\newcommand{\textrmsmall}[1]{\textrm{\fontsize{6}{7.2}\selectfont{#1}}}

\newcommand{\human}{\textrmsmall{H}\xspace}
\newcommand{\robot}{\textrmsmall{R}\xspace}

\newcommand{\action} [2]{\ensuremath{a^{#1}_{#2}}}
\newcommand{\Action} [2]{\ensuremath{A^{#1}_{#2}}}

\newcommand{\ah}[1]{\action{\human}{#1}}
\newcommand{\ar}[1]{\action{\robot}{#1}}
\newcommand{\AH} {\Action{\human}{}}
\newcommand{\AR} {\Action{\robot}{}}

\newcommand{\history} [2]{\ensuremath{h^{#1}_{#2}}}

\newcommand{\x}[1]{\ensuremath{x_{#1}}}

\newcommand{\X}{\ensuremath{X}}
\newcommand{\dH}[1]{\ensuremath{d^{\human}_{#1}}}
\newcommand{\dR}[1]{\ensuremath{d^{\robot}_{#1}}}
\newcommand{\dHn}[1]{\ensuremath{d^{\human'}_{#1}}}
\newcommand{\dRn}[1]{\ensuremath{d^{\robot'}_{#1}}}

\newcommand{\policy} [1]{\ensuremath{\pi^{#1}}}
\newcommand{\policyh}{\policy{\human}\xspace}
\newcommand{\policyr}{\policy{\robot}\xspace}
\newcommand{\policyropt}{\ensuremath{\policy{\robot}_{*}}\xspace}

\tikzset{->-/.style={decoration={
  markings,
  mark=at position 0.5 with {\arrow{#1}}},postaction={decorate}}}
\newcommand{\myarrow}[1]{\tikz{\draw[->-={stealth}, line width=0.4mm, dashed] (0,0) to (0.6cm, 0cm);}}

\newcommand{\intention}[1]{\ensuremath{\theta_{#1}}}

\newcommand{\ahn}[1]{\action{\human'}{#1}}
\newcommand{\arn}[1]{\action{\robot'}{#1}}

\newcommand{\rwdR}[1]{\ensuremath{r^\robot_{#1}}}
\newcommand{\rwdH}[1]{\ensuremath{r^\human_{#1}}}
\newcommand{\RwdR}[1]{\ensuremath{R^\robot_{#1}}}
\newcommand{\RwdH}[1]{\ensuremath{R^\human_{#1}}}

\newcommand{\spdR}[1]{\ensuremath{v^\robot_{#1}}}
\newcommand{\spdH}[1]{\ensuremath{v^\human_{#1}}}
\newcommand{\spdRn}[1]{\ensuremath{v^{\robot'}_{#1}}}
\newcommand{\spdHn}[1]{\ensuremath{v^{\human'}_{#1}}}

\newcommand{\obsD}[1]{\ensuremath{o_{#1}}}

\newcommand{\vh}[1]{\ensuremath{v^\human_{#1}}}
\newcommand{\vr}[1]{\ensuremath{v^\robot_{#1}}}

\newcommand{\gmean}{\ensuremath{\mu}}

\newcommand{\gcov}{\ensuremath{k}}
\newcommand{\gx}[1]{\ensuremath{\mathbf{x}_{#1}}}
\newcommand{\glen}[1]{\ensuremath{\sigma_{#1}}}

\newcommand{\betaD}{\ensuremath{Beta}\xspace}

\newcommand{\safef}{\ensuremath{p^\textrmsmall{G}}\xspace}
\newcommand{\safefp}{\ensuremath{\tilde{p}^\textrmsmall{G}}\xspace}

\newcommand{\hData}{\ensuremath{\mathbb{D}}\xspace}

\newcommand{\tildeV}{\ensuremath{\tilde{V}}\xspace}

\newcommand{\safeVStar}{\ensuremath{\tilde{V}^{\textrmsmall{G}^*}}\xspace}
\newcommand{\safeConstant}{\ensuremath{\phi}\xspace}

\newcommand{\rb}{\ensuremath{r^\textrmsmall{B}}\xspace}
\newcommand{\disRB}{\ensuremath{\beta}\xspace}

\newcommand{\optply}{\ensuremath{\mathcal{A}}\xspace}

\newcommand{\ipomg}{IPL-G\xspace}
\newcommand{\ipom}{IPL\xspace}

%% file: figs/drive-pomdp-lite.tex
\begin{tikzpicture}[align=center]

% time t
\node [s] 				(x) {\x{t}};
\node [a, above right=1cm of x](ar) {\ar{t}};
\node [a, above =of ar](ah) {\ah{t}};
\node [s, above=4cm of x](theta) {\intention{t}};
%\node [s, below left=0.3cm of theta](his) {\history{k}{t}};
\node [s, above right=0.3cm of theta](his) {\history{k}{t}};
\node [s, below=1.5cm of theta] (obs) {\obsD{t}};
% \node [r, below right=0.3cm of x] (r) {\rwd{t}};

%\path [draw, -X] (x) -- (ar);
\path [draw, -X] (x) -- (ah);
\path [draw, -X] (theta) -- (ah);
%\path [draw, -X] (ar) -- (ah);
%\path [draw, -X] (his) edge [bend right=15] (ah);
\path [draw, -X] (his) -- (ah);
\path [draw, -X] (theta) -- (obs);
\path [draw, -X] (x) -- (obs);
% \path [draw, -X] (x) -- (r);

% time t+1
\node [s, right=3cm of x] 				(xn) {\x{t+1}};
\node [s, above=4cm of xn](thetan) {\intention{t+1}};
\node [s, below=1.5cm of thetan] (obsn) {\obsD{t+1}};
% \node [r, below right=0.3cm of xn] (rn) {\rwd{t+1}};

\path [draw, -X] (x) -- (xn);
%\path [draw, -X] (x) edge [bend right] (obsn);
\path [draw, dashed, -X] (theta) -- (thetan);
\path [draw, -X] (xn) -- (obsn);
\path [draw, -X] (ah) -- (obsn);
\path [draw, -X] (ar) -- (obsn);
\path [draw, -X] (thetan) -- (obsn);
\path [draw, -X] (ah) -- (xn);
\path [draw, -X] (ar) -- (xn);
% \path [draw, -X] (xn) -- (rn);
\end{tikzpicture}

%% file: figs/drive-flowchart.tex
\begin{tikzpicture}[align=center]

\node [b] (robot) {Robot};‎
\node [b, below=of robot, yshift=10pt] (human) {Human};‎
\node [b, below=of human, yshift=-12pt] (env) {Environment};‎
\begin{scope}[on background layer]
% \node [b, fit=(robot)(human), fill=black!10!white, inner sep=5pt, label=Team] (team) {};
%\node [b, fit=(robot)(human), fill=black!10!white, inner xsep=5pt, inner ysep=16pt, label=Team] (team) {};
\node [b, fit=(robot)(human), fill=black!10!white, inner xsep=5pt, inner ysep=16pt] (team) {};
\end{scope}

%\path [draw, -X] (robot) -- node [anchor=center, auto] {\ar{t}} (human);
% \path [draw, -X] (human) -- node [anchor=north, auto] {\ah{t}} (team);
% \path [draw, -X] (human) -- node [anchor=center, auto] {\ah{t}} (robot);
% \path [draw, -X] (robot) -- node [anchor=south east, auto] {\ar{t}} (human);
% \path [draw, -X] (human) -- node [anchor=north west, auto] {\ah{t}} (robot);
% \path [draw, -X] (team) -- node [anchor=center] 
% {\ah{t} \ar{t}} (env);
\path [draw, -X] (team) -- node [anchor=center, pos=0.5] 
{\ah{t} \ar{t}} (env);

% min edits
%\draw [o=2pt] (human) -- node [anchor=east, pos=0.5] 
%{\ah{t}} (team);

\coordinate [left=of env] (temp);
\path [draw, -X] (env) -- node [anchor=south, text width=1cm] 
{\x{t+1}} (temp) |- (team);
% \path [draw, -X] (env) -- node [anchor=center, text width=1cm] 
% {\x{t+1} \rwd{t+1}} (temp) |- (team);
\end{tikzpicture}

%% file: wafr18-appendix.tex
\section{Appendix}
\label{sec:appendix}

\subsection{Proof of~\thmref{thm:theorem1}}

%Our intuition is that although safe POMDP-lite
%discourages the robot to explore the unsafe state
%action pairs,
%as long as $\safef(\x{}, \ar{}) > 0$, the
%algorithm will eventually visit those unknown state action 
%pairs and converge.

\thmref{thm:theorem1} is essentially a PAC bound,
and the key to prove it is to show
that at each time step our algorithm is $\epsilon-$optimistic
with respect to the Bayes-optimal policy, and the value
of optimism decays to zero given enough samples of 
unknown state action pairs.

It is obvious that the value of optimism of our algorithm 
(\equref{eq:value-safe})
decays faster than the original POMDP-lite algorithm
(\equref{eq:value}), 
since $0 < \safef(\x{}, \ar{}) < 1$, which means strictly 
less reward bonus.
In addition,
the following lemma states that with proper choice 
of \disRB, our algorithm generates a
value function that is $\epsilon-$optimistic with respect to the 
Bayes-optimal policy.

\begin{lemma}{($\epsilon-$optimistic)}
    Let $\safeVStar(b_t, \x{t})$ be the value function for 
    our algorithm (with reward bonus), and let $V^*(b_t, \x{t})$
    be the Bayes-optimal value function.
    If $\disRB=\frac{|\X{}|^2|\AR|}{\safeConstant(1-\gamma)^2}$, then
    $\forall \x{t}$, $\safeVStar(b_t, \x{t}) \geq V^*(b_t, \x{t}) - \epsilon$.
    \label{lem:lemma1}
\end{lemma}

\begin{proof}{ (\lemref{lem:lemma1}) }
    According to~\equref{eq:value-safe}, we have

    \begin{equation}
        \label{eq:value-safe}
        \begin{split}
            \safeVStar(b_t, \x{t}) = & \underset{\ar{t}}{\max} \bigg{\{} \frac{|\X{}|^2|\AR|}{\safeConstant(1-\gamma)^2} \safef(\x{t}, \ar{t}) \rb(b_t, \x{t}, \ar{t}) +
            \rwdR{}(b_t, \x{t}, \ar{t}) + \\
            & \underset{\x{t+1}}{\sum}
            \gamma P(\x{t+1} | b_t, \x{t}, \ar{t}) \safeVStar(b_{t}, \x{t+1})
            \bigg{\}} \\
            \geq & \underset{\ar{t}}{\max} \bigg{\{} \frac{|\X{}|^2|\AR|}{(1-\gamma)^2} \rb(b_t, \x{t}, \ar{t}) +
            \rwdR{}(b_t, \x{t}, \ar{t}) + \\
            & \underset{\x{t+1}}{\sum}
            \gamma P(\x{t+1} | b_t, \x{t}, \ar{t}) \safeVStar(b_{t}, \x{t+1})
            \bigg{\}} \\
            \geq & V^*(b_t, \x{t}) - \epsilon
        \end{split}
    \end{equation}

    The second line follows from the fact that
    $\safef(\x{t}, \ar{t}) \geq \safeConstant$. 
    The third line follows from Lemma $3$ in~\cref{chen2016pomdp}.
    
\end{proof}

With~\lemref{lem:lemma1}, the analysis in the 
POMDP-lite paper~\cref{chen2016pomdp} can be applied here, 
and thus proves our theorem.

%% file: wafr2018.bbl
\begin{thebibliography}{10}
\providecommand{\url}[1]{\texttt{#1}}
\providecommand{\urlprefix}{URL }

\bibitem{abbeel2004apprenticeship}
Abbeel, P., Ng, A.Y.: Apprenticeship learning via inverse reinforcement
  learning. In: Proc. Int. Conf. on Machine Learning. p.~1. ACM (2004)

\bibitem{astington1993child}
Astington, J.W.: The child's discovery of the mind, vol.~31. Harvard University
  Press (1993)

\bibitem{bai2015intention}
Bai, H., Cai, S., Ye, N., Hsu, D., Lee, W.S.: Intention-aware online pomdp
  planning for autonomous driving in a crowd. In: Proc. IEEE Int. Conf. on
  Robotics \& Automation. pp. 454--460. IEEE (2015)

\bibitem{bandyopadhyay2013intention}
Bandyopadhyay, T., Won, K.S., Frazzoli, E., Hsu, D., Lee, W.S., Rus, D.:
  Intention-aware motion planning. In: Proc. Int. Workshop on the Algorithmic
  Foundations of Robotics, pp. 475--491. Springer (2013)

\bibitem{bratman1987intention}
Bratman, M.: Intention, plans, and practical reason. Harvard University Press,
  Cambridge, MA (1987)

\bibitem{chen2016pomdp}
Chen, M., Frazzoli, E., Hsu, D., Lee, W.S.: Pomdp-lite for robust robot
  planning under uncertainty. In: Proc. IEEE Int. Conf. on Robotics \&
  Automation. pp. 5427--5433. IEEE (2016)

\bibitem{coulter1992implementation}
Coulter, R.C.: Implementation of the pure pursuit path tracking algorithm.
  Tech. rep., Carnegie-Mellon UNIV Pittsburgh PA Robotics INST (1992)

\bibitem{duff2003design}
Duff, M.O.: Design for an optimal probe. In: Proc. Int. Conf. on Machine
  Learning. pp. 131--138 (2003)

\bibitem{feinfield1999young}
Feinfield, K.A., Lee, P.P., Flavell, E.R., Green, F.L., Flavell, J.H.: Young
  children's understanding of intention. Cognitive Development  14(3),
  463--486 (1999)

\bibitem{fern2007decision}
Fern, A., Natarajan, S., Judah, K., Tadepalli, P.: A decision-theoretic model
  of assistance. In: Proc. Int. Joint Conf. on Artificial Intelligence. pp.
  1879--1884 (2007)

\bibitem{garcia2012safe}
Garcia, J., Fern{\'a}ndez, F.: Safe exploration of state and action spaces in
  reinforcement learning. Journal of Artificial Intelligence Research  45,
  515--564 (2012)

\bibitem{hayward1972near}
Hayward, J.C.: Near miss determination through use of a scale of danger. Tech.
  rep., Pennsylvania State University PA, USA (1972)

\bibitem{kaelbling1998planning}
Kaelbling, L.P., Littman, M.L., Cassandra, A.R.: Planning and acting in
  partially observable stochastic domains. Artificial intelligence  101(1-2),
  99--134 (1998)

\bibitem{kahneman2003maps}
Kahneman, D.: Maps of bounded rationality: Psychology for behavioral economics.
  The American economic review pp. 1449--1475 (2003)

\bibitem{kolter2009near}
Kolter, J.Z., Ng, A.Y.: Near-bayesian exploration in polynomial time. In: Proc.
  Int. Conf. on Machine Learning. pp. 513--520. ACM (2009)

\bibitem{lam2015improving}
Lam, C.P., Yang, A.Y., Driggs-Campbell, K., Bajcsy, R., Sastry, S.S.: Improving
  human-in-the-loop decision making in multi-mode driver assistance systems
  using hidden mode stochastic hybrid systems. In: Proc. IEEE Int. Conf. on
  Intelligent Robots and Systems. pp. 5776--5783. IEEE (2015)

\bibitem{nikolaidis2016formalizing}
Nikolaidis, S., Kuznetsov, A., Hsu, D., Srinivasa, S.: Formalizing human-robot
  mutual adaptation: A bounded memory model. In: Proc. ACM/IEEE Int. Conf. on
  Human-Robot Interaction. pp. 75--82. IEEE Press (2016)

\bibitem{nikolaidis2015efficient}
Nikolaidis, S., Ramakrishnan, R., Gu, K., Shah, J.: Efficient model learning
  from joint-action demonstrations for human-robot collaborative tasks. In:
  Proc. ACM/IEEE Int. Conf. on Human-Robot Interaction. pp. 189--196. ACM
  (2015)

\bibitem{scikit-learn}
Pedregosa, F., Varoquaux, G., Gramfort, A., Michel, V., Thirion, B., Grisel,
  O., Blondel, M., Prettenhofer, P., Weiss, R., Dubourg, V., Vanderplas, J.,
  Passos, A., Cournapeau, D., Brucher, M., Perrot, M., Duchesnay, E.:
  Scikit-learn: Machine learning in {P}ython. Journal of Machine Learning
  Research  12,  2825--2830 (2011)

\bibitem{rasmussen2004gaussian}
Rasmussen, C.E.: Gaussian processes in machine learning. In: Advanced lectures
  on machine learning, pp. 63--71. Springer (2004)

\bibitem{rubinstein1998modeling}
Rubinstein, A.: Modeling bounded rationality. MIT press (1998)

\bibitem{sadigh2016information}
Sadigh, D., Sastry, S.S., Seshia, S.A., Dragan, A.: Information gathering
  actions over human internal state. In: Proc. IEEE Int. Conf. on Intelligent
  Robots and Systems. pp. 66--73. IEEE (2016)

\bibitem{sadigh2016planning}
Sadigh, D., Sastry, S., Seshia, S.A., Dragan, A.D.: Planning for autonomous
  cars that leverage effects on human actions. In: Proc. Robotics: Science \&
  Systems (2016)

\bibitem{sunberg2017value}
Sunberg, Z.N., Ho, C.J., Kochenderfer, M.J.: The value of inferring the
  internal state of traffic participants for autonomous freeway driving. In:
  American Control Conference (ACC). pp. 3004--3010. IEEE (2017)

\bibitem{vert2004primer}
Vert, J.P., Tsuda, K., Sch{\"o}lkopf, B.: A primer on kernel methods. Kernel
  methods in computational biology  47,  35--70 (2004)

\end{thebibliography}
